\documentclass[lettersize,journal]{IEEEtran}
\usepackage{amsmath,amsfonts}
\usepackage{algorithmic}
\usepackage{algorithm}
\usepackage{array}
\usepackage[caption=false,font=normalsize,labelfont=sf,textfont=sf]{subfig}
\usepackage{textcomp}
\usepackage{stfloats}
\usepackage{url}
\usepackage{verbatim}
\usepackage{graphicx}
\usepackage{cite}
\usepackage{times}
\usepackage{epsfig}
\usepackage{bbm}
\usepackage{soul}
\usepackage{pifont}
\usepackage{dsfont}
\usepackage{enumitem}
\usepackage{comment}
\usepackage{amssymb}
\usepackage{booktabs}
\usepackage{multirow}
\usepackage{colortbl}
\usepackage{xcolor}
\usepackage{adjustbox}
\usepackage{threeparttable}  
\usepackage{amsthm}
\usepackage{CJKutf8}
\usepackage[mathscr]{eucal}

\usepackage{hyperref}
\hypersetup{
	colorlinks=true,
	linkcolor=cyan,
	filecolor=blue,      
	urlcolor=red,
	citecolor=green,
}

\usepackage[capitalize]{cleveref}
\crefname{section}{Sec.}{Secs.}
\Crefname{section}{Section}{Sections}
\Crefname{table}{Table}{Tables}
\crefname{table}{Table.}{Tabs.}

\begin{document}

\title{Balanced Destruction-Reconstruction Dynamics for Memory-replay Class Incremental Learning}

\author{Yuhang~Zhou, Jiangchao~Yao, Feng~Hong, Ya~Zhang, and Yanfeng~Wang
\IEEEcompsocitemizethanks{\IEEEcompsocthanksitem 
Y. Zhou, J. Yao, Y. Zhang and Y. Wang are with the Cooperative Medianet Innovation Center, Shanghai Jiao Tong University and Shanghai AI Laboratory, Shanghai 200240, China. (E-mail: \{zhouyuhang, sunarker, ya\_zhang, wangyanfeng\}@sjtu.edu.cn). F. Hong is with the Cooperative Medianet Innovation Center, Shanghai Jiao Tong University, Shanghai 200240, China. (E-mail: feng.hong@sjtu.edu.cn)
.}}

\markboth{Journal of \LaTeX\ Class Files,~Vol.~14, No.~8, August~2023}%
{Shell \MakeLowercase{\textit{et al.}}: A Sample Article Using IEEEtran.cls for IEEE Journals}

\maketitle

\begin{abstract}
Class incremental learning (CIL) aims to incrementally update a trained model with the new classes of samples (plasticity) while retaining previously learned ability (stability). To address the most challenging issue in this goal, \textit{i.e.,} catastrophic forgetting, the mainstream paradigm is memory-replay CIL, which consolidates old knowledge by replaying a small number of old classes of samples saved in the memory.
Despite effectiveness, the inherent destruction-reconstruction dynamics in memory-replay CIL are an intrinsic limitation: if the old knowledge is severely destructed, it will be quite hard to reconstruct the lossless counterpart. Our theoretical analysis shows that the destruction of old knowledge can be effectively alleviated by balancing the contribution of samples from the current phase and those saved in the memory.
Motivated by this theoretical finding, we propose a novel Balanced Destruction-Reconstruction module (BDR) for memory-replay CIL, which can achieve better knowledge reconstruction by reducing the degree of maximal destruction of old knowledge.
Specifically, to achieve a better balance between old knowledge and new classes, the proposed BDR module takes into account two factors: the variance in training status across different classes and the quantity imbalance\footnote{For memory-replay CIL, due to the memory cost and the potential privacy risks, typically, a small number of samples from old classes can be saved.} of samples from the current phase and memory. By dynamically manipulating the gradient during training based on these factors, BDR can effectively alleviate knowledge destruction and improve knowledge reconstruction.
Extensive experiments on a range of CIL benchmarks have shown that as a lightweight plug-and-play module, BDR can significantly improve the performance of existing state-of-the-art methods with good generalization.
Our code is publicly available \href{https://github.com/zyuh/BDR-main/tree/main}{here}.
\end{abstract}

\begin{IEEEkeywords}
Continual Learning, Catastrophic Forgetting, Memory Replay, Data Imbalance.
\end{IEEEkeywords}

\section{Introduction}
\IEEEPARstart{A}{lthough} deep neural networks (DNNs) have achieved success in many computer vision tasks~\cite{he2016deep, long2015fully, ren2015faster, krizhevsky2017imagenet}, most explorations are conducted in a static setting, which cannot promote new knowledge continually acquired from the varying environments. To empower DNNs with this ability, Class incremental learning (CIL) are introduced to modulate the deep models for continually learning from new classes without forgetting the learnt knowledge in old classes~\cite{castro2018end, rebuffi2017icarl}.

\begin{figure}[t]
\centering{\includegraphics[width=0.95\linewidth]{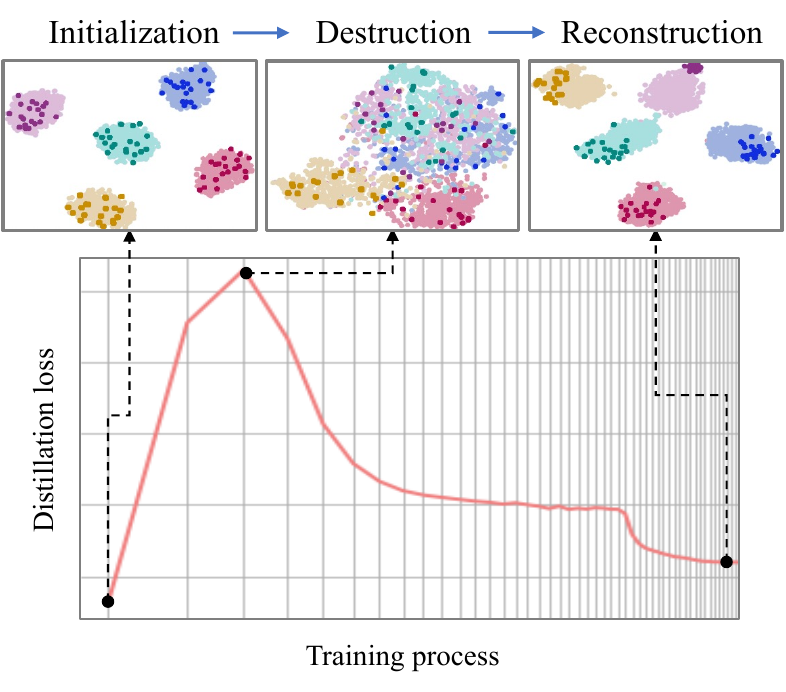}}
\caption{The destruction-reconstruction dynamic in memory-replay CIL. In this figure, we plot the curve of the CwD's distillation loss~\cite{shi2022mimicking} and visualize the embedding distributions of the representative samples from old classes at the initial stage, the peak stage and the converged stage.
In all tSNE visualization, the foreground points are the projection of representative samples and the background points are the footprint of all samples in old classes. We can find that the training goes through the knowledge destruction and the reconstruction \textit{w.r.t.} representative samples, where the representative samples are not representative as the initialization.}
\label{fig1}
\end{figure}

The core challenge in class incremental learning is how to avoid the catastrophic forgetting issue~\cite{goodfellow2013empirical, kirkpatrick2017overcoming, mccloskey1989catastrophic}. To achieve this goal, extensive methods have been explored in the recent years~\cite{ahn2021ss, hou2019learning, hu2021distilling, liu2021adaptive, liu2020mnemonics, tao2020topology, zhu2021class}. 
One line of works can be summarized as a regularization-based way~\cite{kirkpatrick2017overcoming, chaudhry2018riemannian, zenke2017continual, aljundi2018memory}, which consolidates the old knowledge via regularization terms when learning on new classes. 
However, the brute-force regularization directly limits the plasticity of the model and usually results in a poor performance~\cite{hsu2018re, van2019three}.
The architecture-based methods explore the parameter isolation idea to reduce the interference among different phases, which helps avoid catastrophic forgetting by architecture design~\cite{fernando2017pathnet, mallya2018packnet, mallya2018piggyback, serra2018overcoming}. Nevertheless, the  architecture size increases along with incoming classes and the requirement of task identification during inference limits its use. 
Currently, memory-replay CIL~\cite{rebuffi2017icarl, wu2019large, castro2018end, belouadah2019il2m} has become a  popular choice as it achieves an impressive trade-off between plasticity and stability by only saving a small number of representative samples and replaying~\cite{liu2021adaptive,hou2019learning,liu2020mnemonics,shi2022mimicking}.

Nevertheless, Figure~\ref{fig1} shows that memory-replay CIL experiences a destruction-reconstruction dynamic when distilling representative samples of old classes to avoid catastrophic forgetting~\cite{wang2021knowledge,hinton2015distilling,yim2017gift}. 
This destruction-reconstruction dynamic presents an intrinsic limitation: if the old knowledge is severely destructed, \textit{e.g.,} the tSNE characterization at the peak stage in Figure~\ref{fig1}, the model almost completely forgotten how to distinguish the old classes. Then, in the later stage, the model has to re-learn the classification from limited representative samples of old classes together with a large number of samples from the new phase. The resulting problem is that it is quite hard to reconstruct the lossless old knowledge any more, and it becomes harder along with the increase of the destruction.

To deal with this dilemma, we first theoretically characterize that the destruction of the old knowledge can be actually alleviated by balancing the gradient contribution of samples. Then, inspired by this finding, we propose a novel Balanced Destruction-Reconstruction module (BDR) for memory-replay CIL to achieve better knowledge reconstruction by reducing the maximal destruction of old knowledge.
Specifically, the proposed BDR tasks into account two factors: the variance in training status across different classes and the quantity imbalance of samples from the current phase and memory.
By dynamically manipulating the gradient based on above two factors during incremental training, BDR can help the model adjust the degree of maximal destruction for old knowledge while gradually learning new knowledge, achieving a better trade-off between plasticity and stability. 
Experimental results on a large number of CIL benchmarks demonstrate that our method successfully reduces the degree of destruction of old knowledge and achieves better knowledge reconstruction.

In a nutshell, the contribution of this paper can be summarized as follows:
\begin{itemize}
    \item We identify that the hidden destruction-reconstruction dynamics about old knowledge can be an intrinsic limitation to memory-replay CIL methods, and theoretically show that it is possible to alleviate this problem by balancing the contribution of different samples.
    \item We propose a novel Balanced Destruction-Reconstruction module (BDR) to mitigate the negative impact of old knowledge destruction for better knowledge reconstruction by dynamically adjusting the logits based on the variance of training status and the quantity imbalance.
    \item BDR can be easily plugged in the existing state-of-the-art memory-replay CIL methods, and extensive experimental results on a range of CIL benchmarks demonstrate the effectiveness and generalization of BDR.
\end{itemize}

\section{Related Work}
The existing CIL methods can be divided into three categories: \emph{regularization-based} methods, \emph{architecture-based} methods and \emph{memory-replay} methods. 

\emph{Regularization-based methods} mitigate catastrophic forgetting by limiting changes of effective parameters in the training~\cite{chaudhry2018riemannian, zenke2017continual, aljundi2018memory}. EWC~\cite{kirkpatrick2017overcoming} estimated the importance of model parameters via a diagonal approximation of Fisher Information Matrix. SI~\cite{zenke2017continual} accumulated the changes in each parameter during training to evaluate the importance of the parameters more accurately. 
RWalk~\cite{chaudhry2018riemannian} integrated both ideas and used exemplars to further improve the results. 
However, such methods often fail to achieve satisfactory performance because such parameter-level constraints severely limit the plasticity of the model~\cite{hsu2018re,van2019three}.

\emph{Architecture-based methods} construct independent parameters for each phase to avoid interference among classes from different phases~\cite{fernando2017pathnet, mallya2018packnet, mallya2018piggyback, serra2018overcoming, rusu2016progressive}. Such methods are usually used in task incremental learning (TIL) settings~\cite{shi2021continual, tang2021layerwise, wang2021training} that allow task identifiers to be provided at inference time. Abati et al.~\cite{abati2020conditional} used task-specific light-weight gating modules to prevent catastrophic forgetting and dynamically used learned knowledge to improve performance. Rajasegaran et al.~\cite{rajasegaran2019random} used an optimal path selection approach that support parallelism and knowledge exchange between old and new classes to keep the proper balance between old knowledge and new classes. However, the key problem in this direction is that the increasing model size is not practical for long incremental learning sequences.

\emph{Memory-Replay methods} use a memory buffer to save a small proportion of representative samples from old classes and together with samples in the current phase to train the model~\cite{rebuffi2017icarl, wu2019large, castro2018end, belouadah2019il2m}. By replaying samples of old classes, these methods can better alleviate catastrophic forgetting and currently achieve the state-of-the-art on various benchmarks~\cite{liu2021adaptive,hou2019learning,shi2022mimicking}. 
UCIR~\cite{hou2019learning} built the fine-grained regularization, namely cosine normalization, forget-less constraint, and inter-class separation, to mitigate the imbalance of classifier weights. 
PODNet~\cite{douillard2020podnet} proposed a distillation loss that constrains the evolution of the representation which remained stable over long runs of small incremental tasks.
AANet~\cite{liu2021adaptive} explicitly build two types of residual blocks at each residual level for plasticity and stability respectively, and trained aggregation weights to balance them, which introduces additional parameters, but the model size does not continuously increase.
CwD~\cite{shi2022mimicking} effectively regularizes representations of each class to scatter more uniformly, thus mimicking the model jointly trained with all classes.
DER~\cite{yan2021dynamically} developed a dynamically expandable representation and a two-stage strategy to achieve better stability-plasticity trade-off.
DyTox~\cite{douillard2022dytox} used a transformer architecture and specialized each forward of our decoder network on a task distribution through a dynamic expansion of special tokens.
More recently, L2P~\cite{wang2022learning} and DualPrompt~\cite{wang2022dualprompt} methods have also achieved satisfying performance, benefiting from more powerful pre-trained models, i.e., ViT~\cite{dosovitskiyimage}, and prompt learning~\cite{lester2021power}.

Note that, we would like to claim that the exemplar-free class incremental learning methods, such as the regularization-based and architecture-based methods mentioned above are important directions. Actually, they can be combined with the memory-replay CIL for improvement. However, in this paper, we are not to criticize and compare with these methods, but aim to deal with the dilemma exhibited in memory-replay CIL.

\section{Methods}
\subsection{Preliminary}
Assume that we have $\{0,1,\dots,T\}$ training phases, 1 initial phase followed by $T$ incremental phases.
Let $\mathcal{D}_t=\{(x_{t,i},y_{t,i})\}_{i=1}^{N_t}$ be the collection of $N_t$ samples emerging in the $t$-th phase, 
and $K_t$ denotes the corresponding category number. In the $0$-th phase, we train the model $\varTheta_0$ on $\mathcal{D}_0$ using a conventional classification loss. In the subsequent $t$-th phase, memory-replay CIL methods use a memory to save only a small number of representative samples of old classes after training, \textit{i.e.,} $M_{t} = \{(x_{t,i},y_{t,i})\}_{i=1}^{N_{M_{t}}}$($N_{M_{t}}\ll N_t$). Note that, it is impractical to save all samples due to the memory cost, retraining cost and privacy risks. Then, for the next incremental phase, the training dataset after combining is $\widetilde{\mathcal{D}}_{t+1}$, where $\widetilde{\mathcal{D}}_{t+1}=\mathcal{D}_{t+1}\cup M_{0:t}$.

\subsection{Motivation}
\label{mov}
In almost all memory-replay CIL, the loss function ($l$) can be divided into two parts: learning new knowledge ($l_{new}$) and consolidating old knowledge ($l_{old}$), namely, $l=l_{new}+l_{old}$.
During the initial stages of training in the $t$-th incremental phase, $l_{new}$ typically plays a more critical role than $l_{old}$, as $\varTheta_{t}$ is initialized with $\varTheta_{t-1}$, which has already learned the old knowledge, \textit{i.e.,} $l_{old}$ is small enough during this time (as shown in Figure~\ref{fig1}'s ``Intialization'').
The model only starts to re-learn old knowledge when it has been destructed to a certain degree (as shown in Figure~\ref{fig1}'s ``Destruction"). We refer to this turning point as \emph{maximal destruction}.
However, since only a limited number of old samples can be replayed, extremely severe forgetting makes it difficult to reconstruct old knowledge as initialization (as shown in Figure~\ref{fig1}'s ``Reconstruction").
Fortunately, we theoretically find that the upper bound of the maximal destruction can be mitigated by balancing the gradient contributions of samples as in the following theorem.

\newtheorem{thm}{\bf Theorem}
\newtheorem{theo}{\bf Theorem}
\begin{thm}\label{thm1}
Let $F_t^{max}(\theta)$ denotes the ``maximal destruction" in $t$-th phase, which corresponds to the difference between the initial loss and the  highest loss observed during the training process. $L_t$ is the average cross-entropy loss, and suppose all previous phases are fully-learnt, then at the early stage of training in the $t$-th phase, $F_t^{max}(\theta)$ has an upper bound:
\begin{align}
F_t^{max}(\theta)&\le\frac{N_s}{2}\alpha^2[\sigma_m(\sum_{j=1}^{t-1} H_j)]\sum_{s=1}^{N_s}||\nabla L_t(\theta_{s-1})||^2_2 + c, \nonumber
\end{align}
where $\alpha$ is the learning rate, $\sigma_m(\cdot)$ means the maximal eigenvalue, $H_j$ is the Hessian matrix of the $j$-th phase, $N_s$ is the update step from initialization to ``maximal destruction", and $c$ is a constant. An adjustable factor here to affect $F_t^{max}(\theta)$ is $||\nabla L_t(\theta) ||_2^2$, which has a following minimum value:
\begin{align}
||\nabla L_t(\theta) ||_2^2 &\ge \frac{4}{N_t^2}  (\sum_{i=1}^{N_{new}} \nabla l_{ce}(x^i_{new}))( \sum_{i=1}^{N_{old}} \nabla l_{ce}(x^i_{old})). 
\nonumber
\end{align}
Note that the above inequality takes an equal sign when the contributions of both new classes and old classes are equal\footnote{Unlike a conventional upper bound to optimize, here we focus on the equal sign of this lower bound to similarly pursue the infimum of $||\nabla L_t(\theta) ||_2^2$.}.

\end{thm}

Theorem~\ref{thm1} means that reducing $||\nabla L_t(\theta_{s-1})||^2_2$ can help to reduce the upper bound of $F_t^{max}(\theta)$, and it is possible to implement this by balancing the gradient contribution of samples from new classes and old classes. Specially, more balanced the gradients are, smaller $||\nabla L_t(\theta)||^2_2$ is. For the complete proof, we kindly refer readers to the Appendix.

\begin{figure*}[t]
    \centering
    \includegraphics[width=0.97\linewidth]{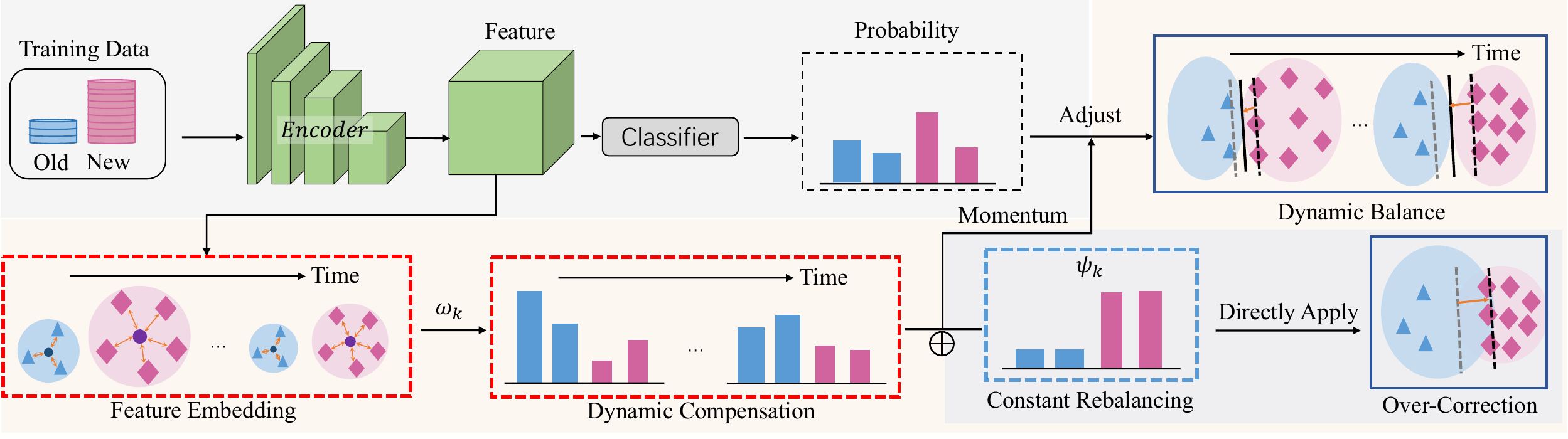}
    \caption{An illustration of BDR, which consists of a constant rebalancing term to eliminate the quantity imbalance of data on the gradient contribution, and a dynamic compensation term to capture the training dynamics and alleviate the over-correction. BDR can effectively alleviate the destruction about old knowledge in memory-replay CIL by manipulating the logit space, yielding a better knowledge reconstruction.}
    \label{method}
\end{figure*}

\subsection{Balanced Destruction-Reconstruction}
\label{3.3}
Inspired by the implication of Theorem~\ref{thm1}, we propose to use the softmax saturation effect~\cite{chen2017noisy, zhang2020class2, li2022long} to balance the gradient contribution of samples.
The principle behind this effect is that the gradient contribution of samples can be influenced by manipulating logits.
To clarify this concept, consider the following gradient of a binary classification as an example:
\begin{align}
\frac{\partial l_{ce}}{\partial z_1}=-\frac{1}{1+e^{z_1-z_2}},
\label{eq:softmax_saturation}
\end{align}
where $z_k$ is the output logit of the $k$-th class.
Eqn.~\eqref{eq:softmax_saturation} shows that as $z_1$-$z_2$ increases, the gradient quickly decays to zero. 
Similarly, we can leverage the manipulation of output logits as a means to adjust the gradient contribution of different samples. 
Concretely, let us use a dynamic offset term $\Delta z_k$ for each class $k$ to influence its gradient and get the modified logits $\widetilde{z_k}$ as follows: 
\begin{align}
\begin{split}
\Delta z_k &\propto \ln(m\psi_k+(1-m)\omega _{k})^\tau, \\
\widetilde{z_k} &= {z_k}+ \Delta z_k, \\
\widetilde{z_1}-\widetilde{z_2} &= (z_1 + \Delta z_1 ) - (z_2 + \Delta z_2).
\label{eq:logit_adjustment}
\end{split}
\end{align}
Here, $m$ is a hyper-parameter, $\psi _k$ represents a constant rebalancing term that addresses the quantity imbalance between new and old classes, and $\omega_{k}$ is a dynamic compensation term that captures the variance of the training status.
A detailed explanation of this design is provided in the following sections.


\par{\noindent \bf On Constant Rebalancing.\hspace{3pt}}
To balance gradient contribution of new classes and old classes, we first need to eliminate the impact of quantity imbalance.
we use Lemma \ref{lemma1} to explain how to achieve this and prove it in the appendix.
\newtheorem{lemma}{Lemma}
\newtheorem{lemmaa}{Lemma}
\begin{lemma}
\label{lemma1}
Let $E_{bal}$ represent the balanced error, which is the average of each per-class error rate, and let $R_{bal}$ represent the balanced risk, given by:
\begin{align}
\begin{split}
E_{bal} &= \frac{1}{K}\sum_{y=1}^K\mathbb{P}_{x|y}(y\notin \mathop{\arg\max}_{{k}'} f_{{k}'}(x)), \\
R_{bal} &= \mathbb{E}[l_{ce}(f(x)+\ln\psi_{},y)],\nonumber
\label{eq:ber}
\end{split}
\end{align}
where $\psi_{}=[\psi_1, \psi_2,..., \psi_K]$ represents the distribution of the sample number of classes, and $\psi_k = \frac{N_k}{\sum_{i=1}^{K}N_i}$ ($N_i$ represents the sample number of $i$-th class and $K$ represents the total number of classes).
The optimal classifier that minimizes $E_{bal}$ is equivalent to the one learned by minimizing $R_{bal}$ ~\cite{menonlong}.
\end{lemma}

Based on Lemma~\ref{lemma1}, we can modify the conventional cross-entropy loss formulation to train a well-balanced classifier for all classes. The updated formula can be expressed as follows:
\begin{equation}
\begin{split}
l_{bal\_ce} = l_{ce}(f(x)+ \ln \psi, y).
\label{eq:new_ce1}
\end{split}
\end{equation}
When using Eqn.~\eqref{eq:new_ce1} for training, new classes will have a larger constant rebalancing value due to their quantity advantage\footnote{Typically, the representative samples of old classes saved in the memory are significantly less than the new classes of samples in terms of the quantity.}, i.e., $\psi_{new} > \psi_{old}$, which will increase the difference $\widetilde{z}_{new}-\widetilde{z}_{old}$ following Eqn.~\eqref{eq:logit_adjustment}. Then, considering the softmax saturation effect, the gradient contribution of new classes will be suppressed, and the more severe the quantity imbalance is, the more obvious this suppression will be.
In total, the constant rebalancing term $\ln\psi_{k}$ works by suppressing the dominant gradient of new classes to promote a balancing trend.

\par{\noindent \bf Limitation.\hspace{3pt}}
\label{cr_limit}
However, purely using Eqn.~\eqref{eq:new_ce1} might not be very reliably, which we utilized CwD~\cite{shi2022mimicking} as the base method to characterize the use of Eqn.~\eqref{eq:new_ce1} with ``CwD w constant rebalancing (CR)".  According to the results in Table~\ref{table_limit0}, the impact of initialization ($\varTheta_{t}$ is initialized with $\varTheta_{t-1}$) is reflected in the significantly smaller loss of old classes compared to new classes, with the loss of old classes being one order of magnitude smaller than that of new classes. And applying Eqn.~\eqref{eq:new_ce1} results in a notable improvement in the final accuracy of old classes, indicating that CR is helpful in reconstructing old knowledge more effectively.
However, the worse convergence loss and final accuracy of new classes suggest that using CR comes at the expense of worse training for new classes.
We refer to this issue as the \emph{over-correction}, which requires the introduction of an additional mechanism to alleviate it.

\begin{table}[t]
\caption{The initial loss, convergence loss and final accuracy of old classes and new classes with and without Constant Rebalancing (CR) on CIFAR-100 based on CwD.}
\centering
\resizebox{0.97\linewidth}{!}{
\begin{tabular}{c|c|ccc}
\toprule
Method & Task & Initial & Convergence & Final Acc\\
\midrule
\multirow2{*}{CwD~\cite{shi2022mimicking}}  & $Old$ & 0.9018 & 0.00001 & 67.2\%  \\
& $New$ & 10.5097 & 0.0851 & 82.1\% \\
\midrule
\multirow2{*}{CwD~\cite{shi2022mimicking} w CR}  & $Old$ & 0.9141 & 0.0004 & 75.3\%   \\
& $New$ & 7.6158 & 0.5732 & 57.1\% \\
\bottomrule
\end{tabular}}
\label{table_limit0}
\end{table}

\vspace{0.1cm}
\par{\noindent \bf On Dynamic Compensation.\hspace{3pt}}
To avoid over-correction, we propose to simultaneously consider the training status along with considering quantity imbalance.
This is because at the initialization from previously trained model, it is inherently local optimal for old classes, but non-optimal for new classes, which incurs discrepancy of gradient contribution at the beginning.
Therefore, it is crucial to track the training dynamic starting from the initialization and transform it as a dynamic compensation term for balancing.
Here, we design a compensation $\omega_{k}$ to realize the above intuition as follows:
\begin{align}
\begin{split}
&\Sigma_{W_k} = \mathbb{E}[f(x_{k,i})-\bar f_k]^2,\\
& \omega_{k} = \frac{1\slash \Sigma_{W_k}}{\sum_{k'}1\slash\Sigma_{W_{k'}}},
\label{eq:variance}
\end{split}
\end{align}
where $f(x_{k,i})$ denotes the feature of the $i$-th sample in the $k$-th class and $\bar f_k = \frac{1}{N_{k}}\sum_{i=1}^{N_{k}}f(x_{k,i})$.
The intuition behind this design is using intra-class variance to estimate the training status. 
Due to the initialization and the limited number of old class samples, their intra-class variance is significantly higher than that of new classes, i.e., $\Sigma_{W_{old}} < \Sigma_{W_{new}}$ and $\omega_{old} > \omega_{new}$. 
As training progresses, the difference between the two gradually diminishes. This characteristic makes it neat to complement with $\psi_{k}$ and thus we design the following term:
\begin{align}
\begin{split}
\pi_k = m\psi_{k} &+ (1-m)\omega_{k},\\
\Delta z_k &\propto \ln(\pi_k)^\tau,
\label{eq12}
\end{split}
\end{align}
where $m$ and $\tau$ are hyper-parameters.
During the initial training phase, the over-correction effect is suppressed due to $\omega_{old} > \omega_{new}$ and $\psi_{new} > \psi_{old}$, allowing the model to learn new knowledge normally. As training progresses, $\omega_{old} - \omega_{new}$ gradually decreases, reducing the inhibitory effect on over-correction. The over-correction effect is then forced to prevent the model from being biased towards new classes.

\vspace{0.1cm}
\par{\noindent \bf The Momentum Update.\hspace{3pt}}
To avoid the variance fluctuations caused by frequent feature updating during training, we adopt a momentum update mechanism on $\bar f_k$ and $\Sigma_{W_k}$ to stabilize the update of $\omega_{k}$ ($\bar f_k^{'}$, $\Sigma_{W_k}^{'}$ and $\omega_{k}^{'}$ represent momentum updated version of $\bar f_k$, $\Sigma_{W_k}$ and $\omega_{k}$ respectively):
\begin{align}
\begin{split}
\bar f_k^{'} = \frac{n}{n+n_k}\bar f_k &+ \frac{1}{n+n_k}\sum_{i=1}^{n_k}f(x_{k,i}), \\
\Sigma_{W_k}^{'} = \frac{n}{n+n_k}\Sigma_{W_k} &+ \frac{1}{n+n_k}\sum_{i=1}^{n_k}(f(x_{k,i})-\bar f_k^{'})^2, \\
\omega_{k}{'} &= \frac{1\slash \Sigma_{W_k}^{'}}{\sum_{k'}1\slash\Sigma_{W_{k'}^{'}}},
\label{eq:momentum1}
\end{split}
\end{align}
where $n_k$ is the number of samples of $k$-th class in the current mini-batch, $n$ is the total number of samples of $k$-th class in the current training set $\widetilde{\mathcal{D}}_t$.
With the momentum version $\omega_{k}{'}$, the linear combination of two balancing factors are formulated as follows:
\begin{align}
\pi_k{'} &= m'\psi_{k} + (1-m')\omega_{k}{'}.
\label{eq:momentum2}
\end{align}
Note that, $m'$ is a hyper-parameter used during training, while $m$ mentioned in Eqn.~\eqref{eq12} is only used in initialization. Similarly, we also apply the momentum mechanism on $\pi_k$, and the final form of the dynamic offset term can be written as follows:
\begin{align}
\begin{split}
\hat \pi_k &=\beta \pi_k+(1-\beta)\pi_k{'},
\label{eq:momentum3}
\end{split}
\end{align}
where $\beta$ is a momentum parameter. The final loss with BDR maintains the same form as Eqn.~\eqref{eq:new_ce1} as follows,
\begin{align}
l_{BDR}= l_{ce}(f(x)+ \ln \hat \pi, y).
\label{BDR}
\end{align}

\vspace{0.1cm}
\par{\noindent \bf Discussion.\hspace{3pt}}
The purpose of BDR is to reduce the upper bound of ``maximal destruction" by alleviating the imbalance in gradient contribution among different classes based on Theorem~\ref{thm1}. In addition to quantity imbalance, we also consider variance compensation for the imbalance caused by initialization in incremental learning and manipulate the gradient using softmax saturation effect. Reducing ``maximal destruction" means less forgetting of old knowledge and easier knowledge reconstruction. Note that, BDR is a lightweight plug-and-play module that can be integrated into existing SOTA methods without increasing the number of parameters during inference.

\begin{table*}[t]
\centering
\caption{Average incremental accuracies (\%) of state-of-the-art methods on CIFAR100-B0 and CIFAR100-B50. $B$ and $S$ denote the number of classes learned at initial phase and per incremental phase respectively. The $\star$ symbol means that the published version of~\cite{douillard2022dytox} has the bug of inconsistent sample selection on different gpus, and the reported scores here are corrected by the author on github. The $+$ symbol means the mixup~\cite{zhangmixup} version of~\cite{douillard2022dytox}.  DER w/o P means DER without pruning~\cite{yan2021dynamically}.}
\resizebox{0.97\linewidth}{!}{
\begin{tabular}{c|cc|cc|cc|cc|cc|cc}
\toprule[0.8pt]
\multirow{2}{*}{Method} & \multicolumn{6}{c}{CIFAR100 (B=0)} & \multicolumn{6}{c}{CIFAR100 (B=50)} \\ \cmidrule(r){2-7} \cmidrule(r){8-13} 
&\#p& Avg (S=10) &\#p& Avg (S=5) &\#p& Avg (S=2) &\#p& Avg (S=25) &\#p& Avg (S=10) &\#p& Avg (S=5) \\
\midrule
Bound & 11.22&80.41 & 11.22&81.49 & 11.22&81.74 & 11.22&77.22 & 11.22&79.89  & 11.22&79.91\\ 
\midrule
iCaRL~\cite{rebuffi2017icarl} & 11.22&65.27$_{\pm1.02}$ & 11.22&61.20$_{\pm0.83}$ & 11.22&56.80$_{\pm0.83}$ & 11.22&71.33$_{\pm0.35}$ & 11.22&65.06$_{\pm0.53}$  &  11.22&58.59$_{\pm0.95}$\\ 
UCIR~\cite{hou2019learning} & 11.22&58.66$_{\pm0.71}$ & 11.22&58.17$_{\pm0.30}$ & 11.22&56.86$_{\pm3.74}$ & 11.22&67.21$_{\pm0.35}$ & 11.22&64.28$_{\pm0.19}$  &  11.22&59.92$_{\pm2.4}$\\ 
BiC~\cite{wu2019large} & 11.22&68.80$_{\pm1.20}$ & 11.22&66.48$_{\pm0.32}$ & 11.22&62.09$_{\pm0.85}$ & 11.22&72.47$_{\pm0.99}$ & 11.22&66.62$_{\pm0.45}$  &  11.22&60.25$_{\pm0.34}$\\ 
WA~\cite{zhao2020maintaining} & 11.22&69.46$_{\pm0.29}$ & 11.22&67.33$_{\pm0.15}$ & 11.22&64.32$_{\pm0.28}$ & 11.22&71.43$_{\pm0.65}$ & 11.22&64.01$_{\pm1.62}$  &  11.22&57.86$_{\pm0.81}$\\ 
PODNet~\cite{douillard2020podnet} & 11.22&58.03$_{\pm1.27}$ & 11.22&53.97$_{\pm0.85}$ & 11.22&51.19$_{\pm1.02}$ & 11.22&71.30$_{\pm0.46}$ & 11.22&67.25$_{\pm0.27}$  &  11.22&64.04$_{\pm0.43}$\\ 
RPSNet~\cite{rajasegaran2019adaptive} & 56.5&68.6 & - & - & - & -  &  -& - & -  &  -& - & -  \\ 
DyTox$^\star$~\cite{douillard2022dytox} & 10.73&71.50 & 10.74& 68.86 & 10.77& 64.82 & - & -  &  -& - & -  &  -\\ 
DyTox+$^\star$~\cite{douillard2022dytox} & 10.73&74.10 & 10.74&71.62& 10.77&68.90 & - & -  &  - & - & -  &  -\\
DER w/o P~\cite{yan2021dynamically} & 61.6&75.36$_{\pm0.36}$ & 117.6&74.09$_{\pm0.33}$ & 285.6&72.41$_{\pm0.36}$ & 22.4&74.61$_{\pm0.52}$ & 39.2&73.21$_{\pm0.78}$  &  67.2&72.81$_{\pm0.88}$\\ 
\rowcolor{gray!20}DER~\cite{yan2021dynamically} & 4.96&74.64$_{\pm0.28}$ & 7.21&73.98$_{\pm0.36}$ & 10.15&72.05$_{\pm0.55}$ & 3.90&74.57$_{\pm0.42}$ & 6.13&72.60$_{\pm0.78}$  &  8.79&72.45$_{\pm0.76}$\\ 
\midrule
\rowcolor{gray!20}DER w BDR & 4.96& 75.69$_{\pm0.02}$& 7.21 & 75.20$_{\pm0.06}$ & 10.15& 74.13$_{\pm0.04}$ & 3.90& 75.69$_{\pm0.09}$ & 6.13& 73.77$_{\pm0.21}$ & 8.79& 72.82$_{\pm0.03}$\\ 
\bottomrule[0.8pt]
\end{tabular}}
\label{table1_sota_cifar}
\end{table*}

\begin{table*}[t]\footnotesize
\centering
\caption{Average incremental accuracies (\%) of state-of-the-art methods on Imagenet-Subset and Imagenet. $B$ and $S$ denote the number of classes learned at initial phase and per incremental phase respectively. The $\star$ symbol means that the published version of~\cite{douillard2022dytox} has the bug of inconsistent sample selection on different gpus, and the reported scores here are corrected by the author on github. The $+$ symbol means the mixup~\cite{zhangmixup} version of~\cite{douillard2022dytox}.  DER w/o P means DER without pruning~\cite{yan2021dynamically}. }

\resizebox{0.9\linewidth}{!}{
\begin{tabular}{c|ccccc|ccccc}
\toprule[0.8pt]
\multirow{3}{*}{Method} & \multicolumn{5}{c}{ImageNet-Subset (B=0, S=10)} & \multicolumn{5}{c}{ImageNet (B=0, S=100)} \\ \cmidrule(r){2-6} \cmidrule(r){7-11}
& \multirow{2}{*}{\#p} & \multicolumn{2}{c}{top-1} & \multicolumn{2}{c}{top-5} &  \multirow{2}{*}{\#p} & \multicolumn{2}{c}{top-1} & \multicolumn{2}{c}{top-5}  \\
\cmidrule(r){3-6} \cmidrule(r){8-11} 
& & Avg & Last & Avg & Last & & Avg & Last & Avg & Last \\

\midrule
Bound & 11.00 & - & 79.12 & - & 93.48 & 11.35 &- & 73.58 & -& 90.60  \\ 
\midrule
E2E~\cite{castro2018end} & 11.22 & - & - & 89.92 & 80.29  &11.68 &- & - & 72.09 & 52.29 \\ 
simple-DER~\cite{li2021preserving} & - &-&-&89.92&80.29  & 28.00 & 66.63& 59.24 &85.62 & 80.76 \\
iCaRL~\cite{rebuffi2017icarl} & 11.22 &-&-&83.60&63.80  &11.68&38.40&22.70&63.70&44.00\\
UCIR~\cite{hou2019learning}&- & - & -&-&-     &- & - & -&-&-\\
BiC~\cite{wu2019large} & 11.22 &-&-&90.60&84.40&11.68&-&-&84.00&73.20\\
WA~\cite{zhao2020maintaining} & 11.22 &-&-&91.00&84.10 &11.68&65.67&55.60&86.60&81.10\\
RPSNet~\cite{rajasegaran2019adaptive} &-&-&-&87.90&74.00&- & - & -&-&-\\
DER~\cite{yan2021dynamically} &7.67 & 76.12 & 66.06 & 92.79 & 88.38 & 14.52 &66.73  &58.62 &87.08 & 81.89\\
DER w/o P~\cite{yan2021dynamically} & 61.6 & 77.18 & 66.70 & 93.23 & 87.52  &  61.6 & 68.84 & 60.16& 88.17& 82.86\\
DyTox$^\star$~\cite{douillard2022dytox} & 11.01& 73.96 & 62.20 & 91.29 & 85.60  & 11.36 & - & - &- & -\\
\rowcolor{gray!20}DyTox+$^\star$~\cite{douillard2022dytox} & 11.01& 77.15 & 67.70 & 93.17 & 89.42  & 11.36 & 70.88 & 60.00 & 90.53 & 85.25\\
\midrule
\rowcolor{gray!20}DyTox+ w BDR & 11.01&  79.10 & 70.28 & 93.39 &  90.08 & 11.36 & 72.06 & 61.48 & 91.98 &  86.30 \\ 

\bottomrule[0.8pt]
\end{tabular}}
\label{table1_sota_imagenet}
\end{table*}

\begin{table*}[t]\footnotesize
\centering
\caption{Average incremental accuracies (\%) of state-of-the-art methods with and without BDR. $B$ and $S$ denote the number of classes learned at initial phase and per incremental phase respectively. All results are reproduced using their public code. We repeated each experiment multiple times with different seeds to compute more reliable results.}

\linespread{2.0}
\resizebox{0.98\textwidth}{!}{
\begin{tabular}{c|cccccccc}
\toprule[0.8pt]
\multirow{2}{*}{Method} & \multicolumn{3}{c}{CIFAR100 (B=50)} & \multicolumn{3}{c}{ImageNet-Subset (B=50)} & 
\multicolumn{2}{c}{ImageNet (B=500)}\\ \cmidrule(r){2-4} \cmidrule(r){5-7} \cmidrule(r){8-9}
& Avg (S=10) & Avg (S=5) & Avg (S=2) & Avg (S=10) & Avg (S=5)  & Avg (S=2) & Avg (S=100) & Avg (S=50)\\
\midrule
UCIR~\cite{hou2019learning} & 66.23$_{\pm0.12}$ & 60.68$_{\pm0.21}$ & 53.11$_{\pm0.43}$ & 70.99$_{\pm0.52}$ & 68.61$_{\pm0.30}$ & 63.92$_{\pm0.33}$ & 61.16$_{\pm0.28}$ & 59.44$_{\pm0.26}$\\ 
\rowcolor{gray!20}+BDR & 68.37$_{\pm0.22}$ & 63.56$_{\pm0.30}$ & 55.71$_{\pm0.58}$ & 73.22$_{\pm0.11}$ & 71.78$_{\pm0.13}$ & 67.30$_{\pm0.15}$ & 68.05$_{\pm0.31}$ & 66.72$_{\pm0.24}$\\ 
\midrule
AANet~\cite{liu2021adaptive} & 69.80$_{\pm0.12}$ & 67.85$_{\pm0.14}$ & 64.90$_{\pm0.21}$ & 71.96$_{\pm0.12}$ & 70.05$_{\pm0.63}$ & 67.28$_{\pm0.34}$ & 64.91$_{\pm0.25}$ & 63.58$_{\pm0.36}$\\ 
\rowcolor{gray!20}+BDR & 70.38$_{\pm0.18}$ & 68.06$_{\pm0.21}$ & 65.00$_{\pm0.18}$ & 73.45$_{\pm0.24}$ & 73.10$_{\pm0.40}$ & 69.82$_{\pm1.14}$ & 68.21$_{\pm0.32}$ & 65.80$_{\pm0.23}$\\ 
\midrule
CwD~\cite{shi2022mimicking} & 67.05$_{\pm0.08}$ & 62.45$_{\pm0.32}$ & 57.34$_{\pm0.45}$ & 71.50$_{\pm0.56}$ & 69.43$_{\pm0.35}$ & 65.14$_{\pm0.12}$ & 52.38$_{\pm0.37}$ & 51.41$_{\pm0.43}$\\ 
\rowcolor{gray!20}+BDR & 69.26$_{\pm0.09}$ & 65.40$_{\pm0.38}$ & 59.69$_{\pm0.32}$ & 74.22$_{\pm0.35}$ & 73.12$_{\pm0.24}$ & 69.68$_{\pm0.38}$ & 61.65$_{\pm0.14}$ & 61.06$_{\pm0.21}$\\ 
\midrule
AANet+CwD~\cite{shi2022mimicking} & 70.13$_{\pm0.09}$ & 68.68$_{\pm0.11}$ & 65.74$_{\pm0.48}$ & 72.92$_{\pm0.29}$ & 71.10$_{\pm0.16}$ & 68.18$_{\pm0.27}$ & 54.04$_{\pm0.29}$ & 51.19$_{\pm0.34}$\\ 
\rowcolor{gray!20}+BDR & 70.79$_{\pm0.11}$ & 69.07$_{\pm0.07}$ & 66.01$_{\pm0.34}$ & 76.05$_{\pm0.15}$ & 73.66$_{\pm0.11}$ & 70.91$_{\pm0.66}$ & 60.86$_{\pm0.45}$ & 59.43$_{\pm0.51}$\\ 
\bottomrule[0.8pt]
\end{tabular}}
\label{table1_orthogonality}
\end{table*}

\section{Experiments}

\subsection{Benchmarks and Implementation Details}
\label{4.1}
\par{\noindent \bf Benchmark Datasets.\hspace{3pt}}
We conduct the experiments on three CIL datasets, \textit{i.e.,} CIFAR-100~\cite{krizhevsky2009learning}, ImageNet-Subset~\cite{rebuffi2017icarl} and ImageNet~\cite{deng2009imagenet}, which are widely adopted in early explorations. Specifically,
CIFAR-100 comprises 60,000 32 $\times$ 32 color image samples spread across 100 classes, with each class consisting of 500 training and 100 test samples. 
ImageNet contains approximately 1.3 million 224 $\times$ 224 color image samples spread across 1,000 classes, with approximately 1.2 million images allocated for training and 50,000 for validation purposes.
ImageNet-Subset is a subset of ImageNet that includes 100 classes, with its data randomly sampled from ImageNet using seed 1993, following~\cite{liu2021adaptive,hou2019learning,liu2020mnemonics,shi2022mimicking}.

\begin{figure*}[h]
    \centering
    \includegraphics[width=\linewidth]{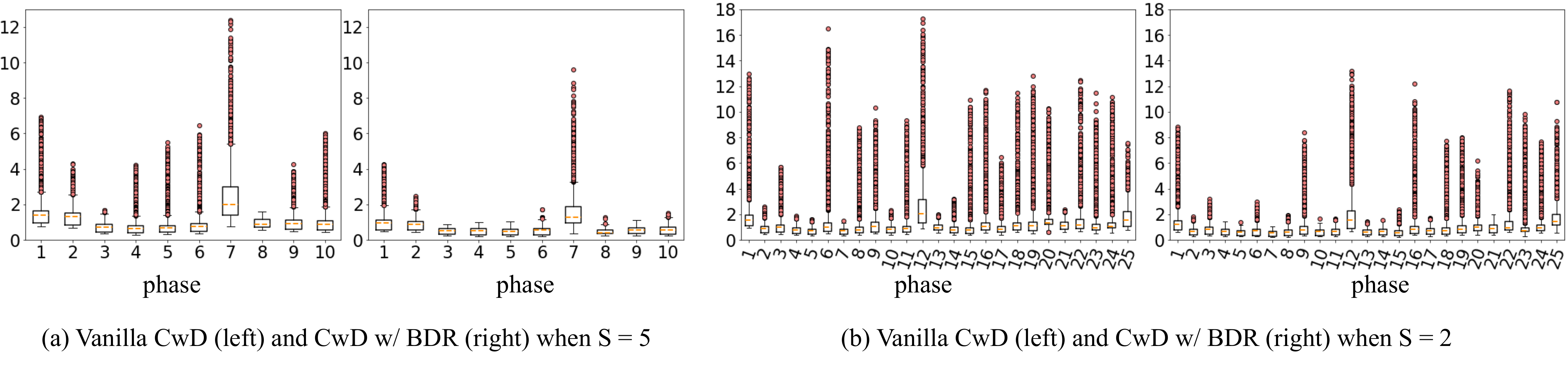}
    \caption{The box-plot of the distillation loss distribution at each phase between the vanilla CwD and CwD with BDR when $S$ = 5 and $S$ = 2. The scattering points outside the quantile range of the median number correspond to the relatively large loss values in the distillation, indicating that the aggressive destruction about the old knowledge during training. 
    Compared with the vanilla CwD, BDR significantly reduces the number of outliers and the value of ``maximal destruction". Besides, BDR has a lower minimum loss, which reflects that the old knowledge after reconstruction is closer to the initialization.}
    \label{fig30}
\end{figure*}

\begin{figure*}[ht]
    \centering
    \includegraphics[width=0.98\linewidth]{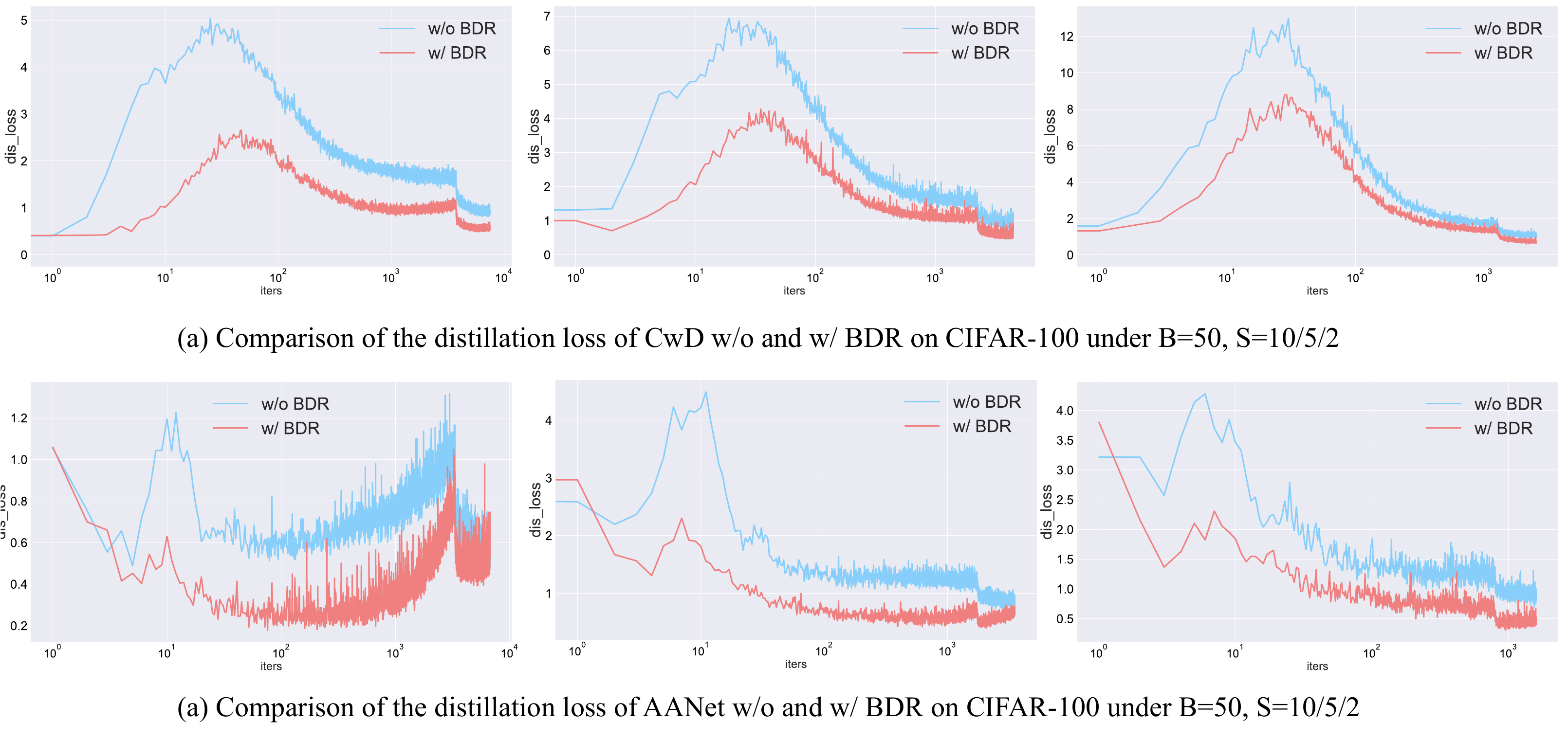}
    \caption{Comparison of the distillation loss at 2-th phase of CwD and AANet w/o and w/ BDR on CIFAR-100 under $B$ = 50, $S$ = 10/5/2. With BDR, the distillation loss will have lower peak values and lower converge values, which means BDR successfully alleviates the degree of maximal destruction and ultimately leads to a higher level of knowledge reconstruction.}
    \label{fig40}
\end{figure*}

\par{\noindent \bf Benchmark Protocols.\hspace{3pt}}
Let $B$ denote the number of classes in the initial phase and $S$ denote the number of classes added in each incremental phase.
For CIFAR-100, we evaluate our methods with protocols: $B$ = 0/50, $S$ = 25/10/5/2. For ImageNet-Subset, we set $B$ = 0/50, $S$ = 10/5/2. For ImageNet, we set  $B$ = 0/500, $S$ = 100/50. 
Specially, when $B$ = 0, we use fixed memory size of 2,000 exemplars in CIFAR-100 and ImageNet-Subset, 20,000 exemplars in ImageNet. When $B$ = 50/500, we use 20 examples as memory per class. These benchmark protocols are consistent with the setting used in the SOTA methods~\cite{yan2021dynamically,douillard2022dytox,shi2022mimicking,liu2021adaptive}.
$\#p$ means the averaging parameter number (by million)  during inference over steps. We report the ``Avg'' accuracy, i.e., average incremental accuracy, which is the average of the accuracies after each phase as defined by~\cite{rebuffi2017icarl}. We also report the final accuracy after the last phase (``Last'').
Note that, if not specified as ``Last'', the defaulted results refer to ``Avg''. We also provide implies that all past data can be accessed during each incremental stage.

\par{\noindent \bf Baselines.\hspace{3pt}}
To show the promise of BDR, we apply it to the following sate-of-the-art baselines:
DER~\cite{yan2021dynamically} and DyTox+~\cite{douillard2022dytox} (the mixup~\cite{zhangmixup} version of DyTox). Other methods including iCaRL~\cite{rebuffi2017icarl}, BiC~\cite{wu2019large}, PODNet~\cite{douillard2020podnet}, WA~\cite{zhao2020maintaining}, UCIR~\cite{hou2019learning} \textit{etc.} are also evaluated for reference. 
In addition, AANet~\cite{liu2021adaptive} and CwD~\cite{shi2022mimicking} which are proposed as the plug-and-play modules in incremental learning, serve as another kind of sate-of-the-art baselines to show that the improvement of our method is orthogonal. For AANet and CwD, we use their UCIR variants~\cite{hou2019learning}. 
It is worth noting that our method does not require special adjustments to the hyper-parameters and original implementation of baselines when plugging in. The default hyper-parameters in BDR are set to $m$ = 0.8, $m'$ = 0.8 and $\beta$ = 0.99. We provide their sensitivity study in Section \ref{4.3}.

\subsection{BDR improves previous SOTA methods}
\label{4.2}
Here, we apply BDR to the two best methods on CIFAR-100 and Imagenet benchmarks to verify that BDR can improve existing SOTA methods. Specifically, for CIFAR-100, DER performed best, and for Imagenet, DyTox performed best. 
As reported in~\cite{douillard2022dytox}, for CIFAR-100, we focus on its ``Avg'' of top-1, while for Imagenet benchmarks, we additional report the performance of ``Last'' and top-5. Since both DyTox and BDR changed the number of parameters, we also marked the comparison of the number of parameters of different methods in Table~\ref{table1_sota_cifar} and Table~\ref{table1_sota_imagenet}.
From Table~\ref{table1_sota_cifar}, we can observe that BDR can improve the average incremental accuracy of DER by about 0.5\% to 2\% on CIFAR-100 and achieve the best results, showing the effectiveness of BDR. Interestingly, the improvement in BDR on DER is even greater than ``DER w/o P", but with far fewer inference parameters.
In addition, we also compare the methods on the Imagenet-subset and Imagenet benchmark in Table~\ref{table1_sota_imagenet}. According to the results in Table~\ref{table1_sota_imagenet}, BDR improves the performance of DyTox+ (the mixup version of DyTox), indicating that BDR is applicable to transformer architectures and that the source of gains is orthogonal to mixup data augmentation and balanced fine-tuning which widely used in incremental learning methods for relieving the imbalance problem.

To further demonstrate the compatibility and orthogonality of BDR, we apply it to other popular memory-replay CIL methods like UCIR~\cite{hou2019learning}, AANet~\cite{liu2021adaptive} and CwD~\cite{shi2022mimicking}, as shown in Table~\ref{table1_orthogonality}. The results show that BDR can improve the average incremental accuracy by approximately 0.5\% to 3\% on CIFAR-100, about 1.5\% to 4.5\% on ImageNet-Subset, and around 2\% to 9.5\% on ImageNet.
Notably, the improvement of BDR on CwD and UCIR is greater than that on AANet and ``AANet+CwD", which may be due to the fact that AANet introduces more learnable parameters, and thus provides a larger search space for plasticity and stability~\cite{liu2021adaptive}. Nevertheless, BDR still achieves significant improvements over AANet on high-resolution datasets (ImageNet-Subset and ImageNet).
On ImageNet, BDR shows a substantial improvement, where BDR shows about 9.5\% performance gain over CwD~\cite{shi2022mimicking}. These results demonstrate the generality of BDR and highlight the importance of analyzing the ``destruction-reconstruction" dynamics, which is widely overlooked by previous methods.

\subsection{The impact of BDR on training}
To understand the working process of BDR, we have plotted the distribution of $L_{old}$ during incremental training for each phase at $S$ = 5/2 in Figure~\ref{fig30}. Note that, for CwD, $L_{old}$ is the distillation loss. Each box in Figure~\ref{fig30} corresponds to the range of $L_{old}$ during an incremental phase, and a higher peak in a box indicates a greater degree of ``maximal destruction" during the training process.
From Figure~\ref{fig30}, we can see that BDR significantly mitigates the degree of ``maximal destruction" according to the reduction of the maximal loss and the outliers. Moreover, the minimum value after convergence is also lowered, which indicates that BDR better reconstructs the old knowledge after mitigating the ``maximal destruction".

Figure~\ref{fig40} presents a more intuitive visualization of the optimization procedure \textit{w.r.t.} $L_{old}$. Although the ``destruction-reconstruction" dynamic persists, BDR significantly reduces the degree of ``maximal destruction," as evidenced by the visibly lower peaks in the loss curve. Additionally, using BDR results in a lower convergence value of $L_{old}$, indicating that BDR promotes more lossless knowledge reconstruction.

\subsection{Ablation study}

\label{4.3}
In this section, we conduct an ablation study to evaluate the influence of each component in BDR, and verify the sensitivity of the corresponding hyper-parameters respectively.

\begin{table}[t]\footnotesize
\caption{Ablations of the different components in BDR, i.e., Constant Rebalancing (CR), Dynamic Compensation (DC.) and Momentum. The experiments are run by ``AANet+CwD'' under $B$ = 50, $S$ = 10 and DER under $B$ = 0, $S$ = 2 on ImageNet-Subset. 
}

\centering
\resizebox{0.99\linewidth}{!}{
\begin{tabular}{c|ccc|c}
\toprule
Method & CR & DC. & Momentum & Avg \\
\midrule
\multirow4{*}{AANet+CwD with}   & & & & 72.92 \\
& $\surd$ & & & 71.68 \\
& $\surd$& $\surd$ & & 75.08 \\
& $\surd$& $\surd$& $\surd$ & 76.05 \\
\midrule
\multirow4{*}{DER with }   & & & &  72.05 \\
& $\surd$ & & &  72.76\\
& $\surd$& $\surd$ & & 73.52 \\
& $\surd$& $\surd$& $\surd$ &  74.13\\
\bottomrule
\end{tabular}}
\label{table_ablation}
\end{table}

\par{\noindent \bf Ablation on the component of BDR.\hspace{3pt}}
As shown in Table~\ref{table_ablation}, we chosen  ``AANet+CwD'' and DER as the base methods and assessed the significance of different components. 
First, according to the performance of ``AANet+CwD" with constant rebalancing (CR), it means that CR does not always lead to the positive improvement, even incurs a decrease of 1.24\% in the average incremental accuracy.
The explanation is attributed to that the constant rebalancing only accounts for quantity imbalance, while ignoring the ``destruction-reconstruction" dynamics in incremental learning.
In comparison, when Dynamic Compensation (DC.) and Momentum were incorporated, accuracy improved significantly, indicating the effectiveness of these component and mechanism and highlighting the potential of BDR.
Specifically, ``CR + DC.'' improved the accuracy of ``AANet+CwD" and DER by 2.16\% and 1.47\%, respectively, while ``CR + DC. + Momentum' improved the accuracy of ``AANet+CwD" and DER by 3.13\% and 2.08\%, respectively.

\par{\noindent \bf Ablation on the sensitivity of hyper-parameters.\hspace{3pt}}
In this section, we explore the impact of various hyper-parameters, including $m$, $m'$ and $\beta$.
As illustrated in Figure~\ref{fig3}, we can see that the better performance is achieved when $m'$ are not equal to be 1, indicating that the dynamic compensation term really plays an indispensable role. Besides, performance also fluctuates as $m'$ changes, which is because existing CIL methods have different trade-offs in plasticity and stability, and components in AANet+CwD may have partially similar spirit. 
In contrast, $m$ has less influence on the results than $m'$, which also confirms the importance of the dynamic compensation component. $m$ and $m'$ can be seen as the interface that adjusts trade-off between stability and plasticity in different methods.
For the momentum parameter $\beta$, an empirical setting of 0.99 may be suitable, as indicated by the results in Table~\ref{table_beta}.

\begin{figure}[t]
    \centering
    \includegraphics[width=0.90\linewidth]{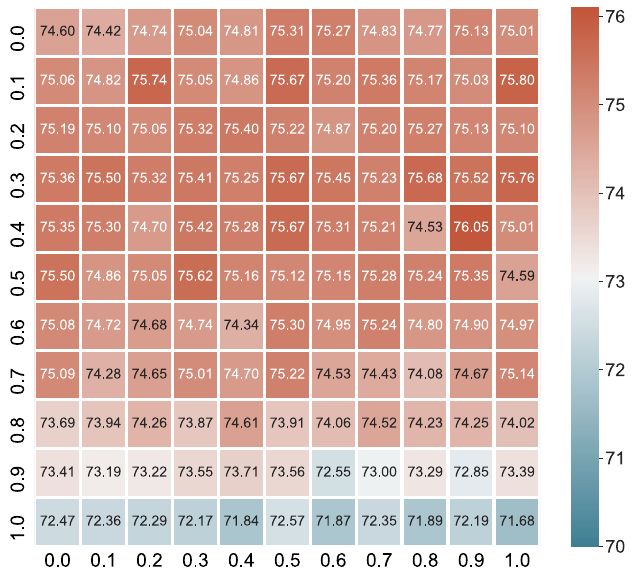}
    \caption{Impact of hyper-parameters $m$ (x-axis) and ${m}'$ (y-axis). The experimental results are obtained on  ImageNet-Subset based on ``AANet+CwD'' with BDR under $B$ = 50, $S$ = 10.}
    \label{fig3}

\end{figure}

\begin{table*}[h!]\small
\centering
\caption{Average incremental accuracies (\%) of state-of-the-art methods with and without BDR. $B$ and $S$ denote the number of classes learned at initial phase and per incremental phase respectively. $R$ means the number of samples saved per class. We repeated each experiment multiple times with different seeds to compute more reliable results.}

\linespread{2.0}
\resizebox{0.95\textwidth}{!}{
\begin{tabular}{c|cccccc}
\toprule[0.8pt]
\multirow{2}{*}{Method} & \multicolumn{3}{c}{CIFAR100 (B=50, R=5)} & \multicolumn{3}{c}{CIFAR100 (B=50, R=40)} \\ \cmidrule(r){2-4} \cmidrule(r){5-7} 
& Avg (S=10) & Avg (S=5) & Avg (S=2) & Avg (S=10) & Avg (S=5) & Avg (S=2) \\
\midrule
UCIR~\cite{hou2019learning} & 55.85$_{\pm0.13}$ & 52.74$_{\pm0.34}$ & 39.54$_{\pm2.82}$ & 68.80$_{\pm0.16}$ & 64.15$_{\pm0.33}$ & 56.88$_{\pm0.12}$\\ 
\rowcolor{gray!20}+BDR & 61.40$_{\pm0.24}$ & 58.98$_{\pm0.37}$ & 46.46$_{\pm0.46}$  & 69.74$_{\pm0.20}$ & 65.95$_{\pm0.10}$ & 57.64$_{\pm0.25}$\\ 
\midrule
AANet~\cite{liu2021adaptive} & 66.07$_{\pm0.06}$ & 62.92$_{\pm0.12}$ & 55.58$_{\pm0.34}$  & 71.68$_{\pm0.30}$ & 70.48$_{\pm0.15}$ & 66.78$_{\pm0.93}$\\ 
\rowcolor{gray!20}+BDR & 67.64$_{\pm0.41}$ & 64.37$_{\pm0.08}$ & 58.17$_{\pm0.30}$  & 71.15$_{\pm0.26}$ & 69.56$_{\pm0.30}$ & 67.02$_{\pm0.07}$\\ 
\midrule
CwD~\cite{shi2022mimicking} & 56.32$_{\pm0.07}$ & 53.87$_{\pm0.09}$ & 39.65$_{\pm2.80}$  & 69.52$_{\pm0.21}$ & 66.63$_{\pm0.12}$ & 59.80$_{\pm0.22}$\\ 
\rowcolor{gray!20}+BDR & 63.33$_{\pm0.38}$ & 60.39$_{\pm0.27}$ & 48.82$_{\pm0.47}$  & 70.53$_{\pm0.05}$ & 68.11$_{\pm0.06}$ & 61.37$_{\pm0.33}$\\ 
\midrule
AANet+CwD~\cite{shi2022mimicking} & 66.05$_{\pm0.12}$ & 63.43$_{\pm0.13}$ & 56.87$_{\pm0.80}$  & 71.62$_{\pm0.29}$ & 70.61$_{\pm0.18}$ & 68.73$_{\pm0.25}$\\ 
\rowcolor{gray!20}+BDR & 67.96$_{\pm0.20}$ & 64.55$_{\pm0.19}$ & 58.35$_{\pm0.62}$  & 71.27$_{\pm0.19}$ & 70.03$_{\pm0.04}$ & 68.23$_{\pm0.26}$ \\ 
\bottomrule[0.8pt]
\end{tabular}}
\label{table1_sota}
\end{table*}

\begin{table}[t]\footnotesize
\caption{Ablation study on the impact of the momentum parameter ($\beta$). These experimental results are obtained on CIFAR-100 based on CwD under $B$ = 50, $S$ = 10 and on Imagenet-Subset based on ``AANet+CwD'' under $B$ = 50, $S$ = 10 respectively. }
\centering
\resizebox{0.95\linewidth}{!}{
\begin{tabular}{l|c|c|c|c}
\toprule
$\beta$ & 0.0 & 0.3 & 0.5 & 0.7   \\
\midrule
CwD & 68.87 & 69.10 & 69.14 & 68.95 \\
AANet+CwD & 75.08 & 75.19 & 74.99 &74.95\\
\midrule
\midrule
$\beta$ & 0.9 & 0.99 & 0.999 & 1.0  \\
\midrule
CwD & 68.91 & 69.24 & 69.09 & 68.90\\
AANet+CwD &  75.17 & 75.54 & 75.36 & 73.60\\
\bottomrule
\end{tabular}}
\label{table_beta}
\end{table}

\begin{figure}[t]
    \includegraphics[width=0.98\linewidth]{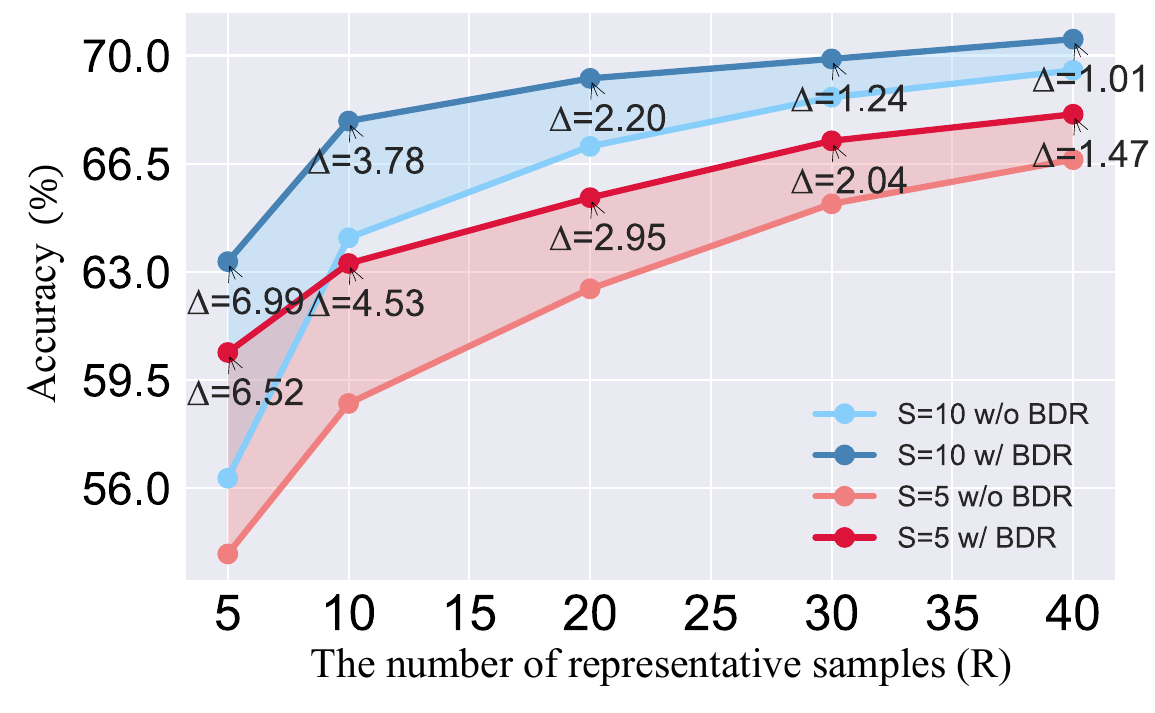}
    \caption{Ablation study on the impact of the number of representative samples ($R$). These experimental results are obtained by changing $R$ on CIFAR-100 under $B$ = 50, $S$ = 10/5.
    }
    \label{table_samplenumber}
    \vspace{-0.15cm}
\end{figure}
\par{\noindent \bf Ablation on the number of representative samples.\hspace{3pt}}
The number of representative samples stored in the memory has a direct impact on the degree of gradient imbalance between new and old classes. To show the robustness of our approach, we varied the number of samples saved per class ($R$) and conducted CwD~\cite{shi2022mimicking} with and without BDR. The results, presented in Figure~\ref{table_samplenumber}, show that BDR significantly improves performance when the number of representative samples is small.
Specially, for $R$ = 5, the average incremental accuracy shows a gain of about 6.99\% and 6.52\% under $S$ = 10 and $S$ = 5, respectively.
Even when more samples are saved (e.g., $R$ = 40), BDR still maintains an improvement of around 1\% to 1.5\%.
These findings suggest that BDR is particularly effective when representative samples are limited.

\par{\noindent \bf Ablation on the compatibility of BDR for more methods.\hspace{3pt}}
To further show the generalizability of our method, we applied BDR to more memory-replay CIL methods like UCIR~\cite{hou2019learning}, AANet~\cite{liu2021adaptive},  CwD~\cite{shi2022mimicking} and ``AANet+CwD''~\cite{shi2022mimicking}, and conducted experiments under $R$ = 5 and $R$ = 40. 
As shown in Table \ref{table1_sota}, BDR achieves a more significant improvement under $R$ = 5 compared to $R$ = 40.
Specifically, the average incremental accuracy shows about 1\% to 9\% and 0.5\% to 1.5\% gain under $R$ = 5 and $R$ = 40 respectively. 
These findings suggest that BDR is effective in the general memory-replay CIL methods but will be more useful in the challenging environments with smaller memory and under more severe imbalance settings.

\subsection{Compare with other class-imbalance methods.}
Some readers may concern the advantages of BDR compared to previous classical class imbalance methods like PaCo~\cite{cui2021parametric}, LA~\cite{menon2020long}, TADE~\cite{zhangself}, BCL~\cite{zhou2022contrastive} and two straightforward baselines \textit{i.e.,} reweighting and balanced sampling. We specially compare them with our proposed BDR in Table~\ref{table3} to demonstrate that our module is more suitable to memory-replay CIL scenarios. It's worth noting that these baselines were originally proposed for static training scenarios, and to ensure fairness, we transplant them to memory-replay class incremental learning scenarios without heavy modification.
As shown in Table~\ref{table3}, except for LA~\cite{menon2020long}, the other methods do not show any performance improvements. This suggests that it's crucial to consider the differences for imbalance learning when facing incremental learning scenarios and static training scenarios. In contrast, BDR is more suitable for incremental learning scenarios and ultimately achieves the best results.

Besides, we should note that there are indeed some explorations about class imbalance in existing incremental learning algorithms. For example, BiC~\cite{wu2019large} introduces bias parameters to correct prediction scores and fine-tunes these parameters using a balanced subset drawn from the training set. UCIR~\cite{hou2019learning} introduces regularization terms, namely cosine normalization, to mitigate the effects of data imbalance. WA~\cite{zhao2020maintaining} alleviates the imbalance by aligning the weight norms of classifiers.
We summarize these methods into two categories: balancing classifier weights and fine-tuning on balanced subsets. 
Despite their potential, these methods failed to analyze how imbalance can lead to catastrophic forgetting during incremental training. Moreover, as demonstrated in Table~\ref{table1_sota_cifar} and Table~\ref{table1_sota_imagenet}, these methods significantly underperform over SOTA methods.

\begin{table}[t]\footnotesize
\caption{Compare with other class-imbalance methods under CIL settings. These experimental results are based on CwD on CIFAR100. $B$ and $S$ denote the number of classes learned at initial phase and per incremental phase respectively.}
\centering
\resizebox{0.99\linewidth}{!}{
\begin{tabular}{c|cccc}
\toprule
\multirow{2}{*}{Method}  & \multicolumn{2}{c}{CIFAR100 (B=50,S=10)} & \multicolumn{2}{c}{CIFAR100 (B=50,S=5)}  \\
\cmidrule(r){2-3} \cmidrule(r){4-5} 
& Avg & Last & Avg & Last  \\
\midrule
CwD~\cite{shi2022mimicking}   & 67.05 & 58.20 & 62.45 & 51.60 \\
w PaCo~\cite{cui2021parametric}   & 54.98 & 40.30 & 50.63 & 45.70 \\
w BCL~\cite{zhou2022contrastive}   & 51.70 & 46.60 & 50.44 & 44.90 \\
w TADE~\cite{zhangself}   & 58.47& 52.50 & 50.33 & 43.80\\
w LA~\cite{menon2020long} & 67.53 & 58.70 & 63.14 & 53.20 \\
w reweighting  & 65.08 & 55.13 & 62.19 & 51.20 \\
w bal. sampling   & 22.21 & 9.15 & 14.65 & 6.53 \\

\midrule
\rowcolor{gray!20}w BDR   & 69.26 & 60.90& 65.40 & 54.30\\
\bottomrule
\end{tabular}}
\label{table3}
\end{table}

\subsection{Limits of only use Constant Rebalancing (CR).}
As discussion in Section~\ref{cr_limit}, CR may lead to the over-correction and hinder the learning of new classes despite mitigating forgetting of old knowledge. 
Conversely, BDR strikes a better balance between the new classes and the old classes, which maintains similar accuracy on old classes as CR, while achieves higher accuracy on new classes, as demonstrated in Table~\ref{table_limit}. In addition, BDR offers some potential flexibility by allowing us to manipulate the values of $m, m'$ and achieve a higher level of stability and plasticity. 
In summary, BDR inherits the advantages of CR, and simultaneously avoids its drawbacks, making the proposed design serve as a promising enhancement for memory-replay class incremental learning.


\begin{figure*}[t]
    \centering
    \includegraphics[width=0.97\linewidth]{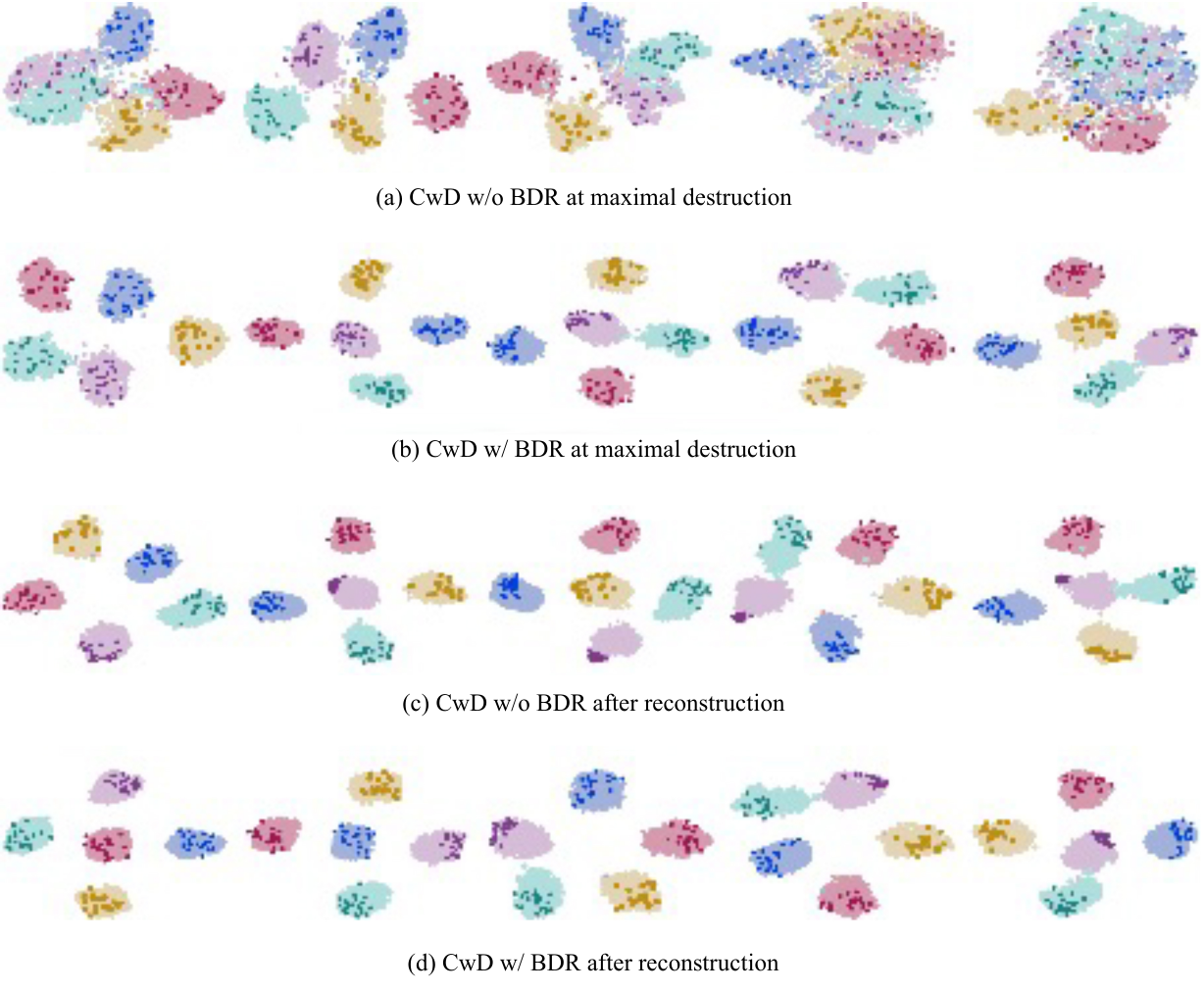}
    \caption{Comparison of t-SNE~\cite{van2008visualizing} about the distribution of old classes on CIFAR-100 ($B$ = 50, $S$ = 10) based on CwD. The training process comprises five incremental phases, each corresponding to a column in the picture.}
    \label{fig6}
\end{figure*}

\begin{table}[t]
\caption{The initial loss, convergence loss and final accuracy of old classes and new classes with Constant Rebalancing (CR) and with BDR on CIFAR-100 based on CwD.}
\centering
\resizebox{0.99\linewidth}{!}{
\begin{tabular}{c|c|ccc}
\toprule
Method & Task & Initial & Convergence & Final Acc  \\
\midrule
\multirow2{*}{CwD}  & $Old$ & 0.9018 & 0.00001 & 67.2\% \\
& $New$ & 10.5097 & 0.0851 & 82.1\% \\
\midrule
\multirow2{*}{CwD w CR}  & $Old$ & 0.9141 & 0.0004 & 75.3\%\\
& $New$ & 7.6158 & 0.5732 & 57.1\% \\
\midrule
\multirow2{*}{CwD w BDR}  & $Old$ & 0.7658 & 0.0003 & 74.4\% \\
& $New$ & 7.9685 & 0.3865 & 61.7\%\\
\bottomrule
\end{tabular}}
\vspace{-0.6cm}
\label{table_limit}
\end{table}

\par{\noindent \bf Visualization results.\hspace{3pt}} 
In Figure~\ref{fig6}, we plot the t-SNE~\cite{van2008visualizing} of CwD~\cite{shi2022mimicking} without and with BDR at the five training phases. 
It can be seen that when the ``maximal destruction'' reaches to, CwD with BDR has better discrimination ability on old classes than the vanilla CwD.
After reconstruction, CwD with BDR has a better reconstruction quality, which is manifested by the smaller distribution shift between foreground points and background points. In other words, the BDR succeeded in preserving old knowledge by reducing the ``maximal destruction", ultimately leading to a higher level of ``reconstruction".


\section{Conclusion}
In this work, we focus on the popular memory-replay class incremental learning paradigm and identify the inherent issue of the destruction-reconstruction dynamics that are widely existed to affect the catastrophic forgetting. Specially we find in this dynamic,
the more severe the destruction is, the more difficult it is to reconstruct old knowledge. However,
our theoretical analysis shows that the maximal destruction can be actually reduced by manipulating the contribution of samples from the current phase and samples from the memory. In this spirit, we propose a BDR module to alleviate the degree of maximal destruction by dynamically balancing the gradient of different classes and ultimately lead to a higher level of knowledge reconstruction. Extensive experiments on the state-of-the-art Class incremental learning methods and multiple datasets demonstrate its effectiveness. 
In the future, we will further explore solutions to the destruction-reconstruction dynamic to better mitigate catastrophic forgetting in CIL, especially in combination with the popular pretraining models.

\newpage

\bibliographystyle{IEEEtran}
\bibliography{egbib}

\newpage
\newpage

{\appendix[Proof]
\begin{theo}
Let $F_t^{max}(\theta)$ denotes the ``maximal destruction" in $t$-th phase, which corresponds to the difference between the initial loss and the  highest loss observed during the training process. $L_t$ is the average cross-entropy loss, and suppose all previous phases are fully-learnt, then at the early stage of training in the $t$-th phase, $F_t^{max}(\theta)$ has an upper bound:
\begin{align}
F_t^{max}(\theta)&\le\frac{N_s}{2}\alpha^2[\sigma_m(\sum_{j=1}^{t-1} H_j)]\sum_{s=1}^{N_s}||\nabla L_t(\theta_{s-1})||^2_2 + c, \nonumber
\end{align}
where $\alpha$ is the learning rate, $\sigma_m(\cdot)$ means the maximal eigenvalue, $H_j$ is the Hessian matrix of the $j$-th phase, $N_s$ is the update step from initialization to ``maximal destruction", and $c$ is a constant. An adjustable factor here to affect $F_t^{max}(\theta)$ is $||\nabla L_t(\theta) ||_2^2$, which has a following minimum value:
\begin{align}
||\nabla L_t(\theta) ||_2^2 &\ge \frac{4}{N_t^2}  (\sum_{i=1}^{N_{new}} \nabla l_{ce}(x^i_{new}))( \sum_{i=1}^{N_{old}} \nabla l_{ce}(x^i_{old})). 
\nonumber
\end{align}
Note the above inequality takes an equal sign when the contributions of both new classes and old classes are equal.

\end{theo}
\begin{proof}\renewcommand{\qedsymbol}{}

The average cross-entropy loss $L_t(\theta)$ of $t$-th phase can be  written as
\begin{align}
    L_t(\theta)=\frac{1}{N_t}\sum_{i=1}^{N_t}l_{ce}(f(x_{t,i}, \theta), y_{t,i}),
\end{align}
where $N_t$ is the number of samples in the current phase. Let $L_j(\theta_j^*)$ denotes the loss values after converge at $j$-th phase, where $\theta_j^*$ is the trained model parameters of $j$-th phase, and let $F_t^j(\theta)$ denotes the average forgetting of $j$-th phase at the current phase, which can be written as
\begin{align}
F_t^j(\theta)=\frac{1}{N_j}\sum_{i=1}^{N_j}l_{ce}(f(x_{j,i}, \theta), y_{j,i})-L_j(\theta_j^*).
\end{align}
Then, the total average forgetting of all old phases at the current phase can be written as
\begin{align}
\begin{split}
F_t^{all}(\theta)&=\sum_{j=1}^{t-1}F_t^j(\theta)\\
&=\sum_{j=1}^{t-1}[\frac{1}{N_j}\sum_{i=1}^{N_j}l_{ce}(f(x_{j,i}, \theta), y_{j,i})-L_j(\theta_j^*)]\\
&=\sum_{j=1}^{t-1}[\frac{1}{N_j}\sum_{i=1}^{N_j}l_{ce}(f(x_{j,i}, \theta), y_{j,i})] + c\\
&=\sum_{j=1}^{t-1}L_j(\theta) + c. 
\end{split}
\end{align}

Since the new model initialized with the old model parameters, at the early stage of training, the total forgetting $F_t^{all}(\theta)$ can be estimated by Second-order Taylor Expansion (STE). For convenience, the constant $c$ is omitted from the following derivation:

\begin{equation}
\begin{split}
F_t^{all}(\theta)=&\sum_{j=1}^{t-1}L_j(\theta)\\
\overset{\text{STE}}\approx& \sum_{j=1}^{t-1}[L_{j}({\theta}^{\ast}_j)+\frac{1}{2}({\theta}-{\theta}_j^{\ast})^TH_j({\theta}-{\theta}_j^{\ast})]\\
=&\sum_{j=1}^{t-1}[L_{j}(\theta^\ast_j)+\frac{1}{2}(\theta+\theta_{t-1}^\ast-\theta_{t-1}^\ast-\theta_j^\ast)^T\\
&H_j(\theta+\theta_{t-1}^\ast-\theta_{t-1}^\ast-\theta_j^\ast)]\\
=&(\theta-\theta^\ast_{t-1})^T\sum_{j=1}^{t-2}[H_j(\theta^\ast_{t-1}-\theta^\ast_j)]\\
&+\frac{1}{2}(\theta-\theta^\ast_{t-1})^T(\sum_{j=1}^{t-1}H_j)(\theta-\theta^\ast_{t-1})+c.
\end{split}
\end{equation}

Considering the lemma 7 in~\cite{liu2022continual}, \textit{i.e.,} if $(\sum^{t-1}_{j=1}H_j)(\theta^*_t-\theta^*_{t-1})=0$ holds for any $t$, then $\sum^{t-1}_{j=1}[H_j(\theta^*_t-\theta^*_{j})]=0$ also holds for any $t$, and ignoring the constant term, we have
\begin{align}
\label{eq18}
    F_t^{all}(\theta)\cong\frac{1}{2}(\theta-\theta^*_{t-1})^T(\sum_{j=1}^{t-1} H_j)(\theta-\theta^*_{t-1}).
\end{align}

Let $s$ denotes the $s$-th gradient update step of parameters from initialization, which can be written as:
\begin{align}
\begin{split}
\theta_s &= \theta_{s-1}-\alpha_s\nabla L_{t}(\theta_{s-1})\\
L_{t}(\theta_s)&=L_{t}(\theta_{s-1})-\alpha_s(\nabla L_t(\theta_{s-1}))^T\nabla L_{t}(\theta_{s-1})
\end{split}
\end{align}

Suppose $N_s$ is the update step from initialization to ``maximal destruction" of all old knowledge and then
\begin{align}
\begin{split}
F_t^{max}(\theta)&\cong\frac{1}{2}(\theta-\theta^*_{t-1})^T(\sum_{j=1}^{t-1} H_j)(\theta-\theta^*_{t-1})\\
&=\frac{1}{2}(\sum_{s=1}^{N_s}\alpha_s\nabla L_t(\theta_{s-1}))^T(\sum_{j=1}^{t-1} H_j)(\sum_{s=1}^{N_s}\alpha_s\nabla L_t(\theta_{s-1}))\\
&\le \sigma_m(\sum_{j=1}^{t-1} H_j)\frac{1}{2}(\sum_{s=1}^{N_s}\alpha_s\nabla L_t(\theta_{s-1}))^T(\sum_{s=1}^{N_s}\alpha_s\nabla L_t(\theta_{s-1}))\\
&\le \sigma_m(\sum_{j=1}^{t-1} H_j)\frac{1}{2}N_s(\sum_{s=1}^{N_s}\alpha_s\nabla L_t(\theta_{s-1}))^T(\alpha_s\nabla L_t(\theta_{s-1}))\\
&= \frac{N_s}{2}\alpha^2[\sigma_m(\sum_{j=1}^{t-1} H_j)]\sum_{s=1}^{N_s}(\nabla L_t(\theta_{s-1}))^T\nabla L_t(\theta_{s-1})\\
&= \frac{N_s}{2}\alpha^2[\sigma_m(\sum_{j=1}^{t-1} H_j)]\sum_{s=1}^{N_s}||\nabla L_t(\theta_{s-1})||^2_2.\\
\end{split}
\end{align}

In conclusion, at the early stage of training, $F_t^{max}(\theta)$ has an upper bound as
\begin{align}
F_t^{max}(\theta)&\le\frac{N_s}{2}\alpha^2[\sigma_m(\sum_{j=1}^{t-1} H_j)]\sum_{s=1}^{N_s}||\nabla L_t(\theta_{s-1})||^2_2 + c, \nonumber
\end{align}

In this formula, we find that $||\nabla L_t(\theta_{s-1})||^2_2$ is a direct factor that affects the upper bound of $F_t^{max}(\theta)$. Thus, reducing $||\nabla L_t(\theta_{s-1})||^2_2$ can help to reduce the ``maximal destruction". 
According to the definition of $L_t$, such a term includes the contributions of both new classes and old classes, which can be formulated as follows:
\begin{align}
\begin{split}
|| \nabla L_t(\theta) ||_2^2 &= \frac{1}{N_t^2}||  \underbrace{\sum_{i=1}^{N_{new}} \nabla l_{ce}(x^i_{new})}_{Con_{new}} + \underbrace{\sum_{i=1}^{N_{old}} \nabla l_{ce}(x^i_{old})}_{Con_{old}}||_2^2.\nonumber
\end{split}
\end{align}
Given the parameter initialization and the severe imbalance between samples of old classes and samples of new classes, the contribution of two terms can be very different.
According to Cauchy's inequality~\cite{sedrakyan2018algebraic}, we have the lower bound
\begin{align}
\begin{split}
||\nabla L_t(\theta) ||_2^2 &\ge \frac{4}{N_t^2}  (\sum_{i=1}^{N_{new}} \nabla l_{ce}(x^i_{new}))( \sum_{i=1}^{N_{old}} \nabla l_{ce}(x^i_{old})) , \nonumber\\
\end{split}
\end{align}
which takes an equal sign when the contributions of both new classes and old classes are equal.
This means that we can puruse the infimum of $||\nabla L_t(\theta)||^2_2$ by manipulating the gradient contribution of different samples. 
Specially, more balanced the gradients are, smaller $||\nabla L_t(\theta)||^2_2$ is.
\end{proof}

\begin{lemmaa}
Let $E_{bal}$ represent the balanced error, which is the average of each per-class error rate, and let $R_{bal}$ represent the balanced risk, given by:
\begin{align}
\begin{split}
E_{bal} &= \frac{1}{K}\sum_{y=1}^K\mathbb{P}_{x|y}(y\notin \mathop{\arg\max}_{{k}'} f_{{k}'}(x)), \\
R_{bal} &= \mathbb{E}[l_{ce}(f(x)+\ln\psi_{},y)],\nonumber
\label{eq:ber}
\end{split}
\end{align}
where $\psi_{}=[\psi_1, \psi_2,..., \psi_K]$ represents the distribution of the sample number of classes, and $\psi_k = \frac{N_k}{\sum_{i=1}^{K}N_i}$ ($N_i$ represents the sample number of $i$-th class and $K$ represents the total number of classes).
The optimal classifier that minimizes $E_{bal}$ is equivalent to the one learned by minimizing $R_{bal}$ ~\cite{menonlong}.
\end{lemmaa}
\begin{proof}\renewcommand{\qedsymbol}{}

Denote the Bayes-optimal classifier that minimizes $E_{bal}$ is $f_{E}^\star\in \arg\min_f E_{bal}$, and the optimal classifier gotten by minimizing $R_{bal}$ is $f_R^\star\in \arg\min_f R_{bal}$. Then
following~\cite{menon2013statistical,menon2020long} and Theorem 1 in~\cite{collell2016reviving},
we have
\begin{align}
\begin{split}
\arg\max_y f_{E}^*(x) &= \arg\max_y \mathbb{P}^{bal}(y|x)\\
&= \arg\max_y \mathbb{P}(x|y).\\
\end{split}
\end{align}

On the other hand,we have
\begin{align}
    \begin{split}
        \arg\max_y \exp(f_R^\star(x))\cdot\mathbb{P}(y) &= \arg\max_y f_R^\star(x)+\ln\psi \\
        &= \arg\max_y \mathbb{P}(y|x) \\
        &= \arg\max_y \mathbb{P}(x|y)\mathbb{P}(y)
    \end{split}
\end{align}
where the second line uses Lemma 1 in~\cite{yu2018learning}, \textit{i.e.,} to output the underlying class probability $\mathbb{P}(y|x)$. Then, we have
\begin{align}
    \begin{split}
    \arg\max_y f_R^\star(x) &= \arg\max_y \exp(f_R^\star(x)) \\
    &= \arg\max_y \mathbb{P}(x|y) \\
    &= \arg\max_y f_E^\star(x)
    \end{split}
\end{align}

which means that the optimal classifiers are equivalent~\cite{menon2020long}.

\end{proof}
}

\vfill

\end{document}